\documentclass[11pt,a4paper]{article}

\usepackage[T1]{fontenc}
\usepackage[utf8]{inputenc}
\usepackage[a4paper,margin=2.6cm]{geometry}
\usepackage{amsmath,amssymb,amsthm}
\usepackage{graphicx}
\usepackage{booktabs}
\usepackage{enumitem}
\usepackage{fancyhdr}
\usepackage[colorlinks=true,linkcolor=black,urlcolor=black,citecolor=black]{hyperref}

\newtheorem{theorem}{Theorem}
\newtheorem{proposition}{Proposition}
\newtheorem{lemma}{Lemma}
\newtheorem{corollary}{Corollary}
\theoremstyle{remark}
\newtheorem{remark}{Remark}

\title{\Large\bfseries The Price of Self-Calibration:\\
Exact Evidence Budgets and Manufactured Blind Sets\\ in Adaptive Monitoring}
\author{Abdou-Raouf Atarmla\\[2pt]
\small Institut National des Postes et T\'el\'ecommunications, Rabat, Morocco\\
\small Togo AI Lab, Lom\'e, Togo\\
\small\texttt{atarmla.abdouraouf@ine.inpt.ac.ma}\\
\small\texttt{achilleatarmla@gmail.com}}
\date{\small August 2026}

\begin{document}
\maketitle

\begin{abstract}
Self-calibrating monitors are the standard answer to fault detection without
ground truth: by continuously adapting their alarm threshold, they guarantee a
prescribed long-run false-alarm rate under arbitrary distribution drift. We
study the price of this guarantee, and we state every law together with its
exact domain of validity. First, the guarantee is an accounting identity,
insensitive to what the monitor is supposed to detect. Two evidence identities
make the cost exact for the online quantile tracker: a persistent step of
height $\delta$ yields a total excess alarm mass within one alarm of
$\delta/\eta$, where $\eta$ is the adaptation gain, and \emph{exactly}
$\delta/\eta$, pathwise, whenever $\delta$ is a lattice multiple of $\eta$;
a ramp of slope $c$ yields a stationary excess rate of exactly $c/\eta$,
independent of accumulated size, up to a boundary $c=\eta(1-\alpha)$ that
coincides exactly with the alarm-rate cap. Second, the fluctuation of the
certificate itself obeys an exact law: the alarm rate over a window of length
$L$ has standard deviation of order $1/L$, not the binomial $1/\sqrt L$,
because the windowed mass telescopes to a difference of the tight internal
state; the closed-form constant is validated to within $3\%$ with no fitted
parameter. Consequently downstream detectors calibrated on the binomial scale
are miscalibrated by the factor $\sqrt{\eta\varphi(q_0)L}$, conservative and
blind in the same proportion; correct calibration shortens detection windows
from quadratic to linear in the inverse fault speed. Third, any monitor
required to tolerate a drift class $\mathcal D$ is blind, at any horizon and
for any decision rule, to every fault in the difference set
$\mathcal D-\mathcal D$; the proof is a deliberately elementary two-point
argument, and the contribution is the object it identifies: for speed-bounded
classes the blind set is exactly the doubled-speed class, and the tracker
absorbs a speed class fixed by its own gain, so that under a certification
regime that declares absorbed drift normal, the monitor manufactures
$\mathcal D$. An exact Gaussian projection bound, sharper
than Pinsker and never vacuous, quantifies power outside the blind set.
Experiments verify each identity to its stated precision, including
per-trajectory exactness for lattice steps, the saturation boundary, the
fluctuation law across twelve configurations, and the speed frontier with
confidence intervals. The results delimit what a clean certificate of a
self-calibrating monitor does and does not certify.
\end{abstract}

\section{Introduction}

Modern monitoring pipelines increasingly operate without ground truth. A
deployed component produces a stream of scores, residuals or nonconformity
measures; a monitor watches the stream and raises alarms; no oracle ever
reveals which alarms were right. In this regime the only guarantee that can be
enforced online is a false-alarm guarantee, and the standard mechanism is
self-calibration: the monitor adapts its threshold so that the empirical alarm
rate tracks a prescribed level $\alpha$. Adaptive conformal inference
\cite{gibbs2021} is the best-known modern instance; self-starting control
charts and sliding-window thresholds are older members of the same family.
These mechanisms work as advertised: the false-alarm rate is controlled under
arbitrary, even adversarial, distribution shift.

This paper is about what the guarantee costs. The mechanism that keeps the
false-alarm rate pinned at $\alpha$ is the same mechanism that absorbs slow
change, whatever its cause. A monitor that must tolerate benign drift will
tolerate a fault that resembles benign drift, and its certificate, the
empirical alarm rate, will remain exactly on target while the fault passes.
Recent empirical work exhibits exactly this failure mode: in
world-model-based self-monitoring of reinforcement learning agents under
continuous observation drift, a sharp threshold separates drifts absorbed as
normal variation from drifts that trigger detection, and fragile systems
collapse before any detector fires \cite{gradual}. Our aim is to characterize the
phenomenon: to replace the observation that adaptive monitors can be fooled
by slow faults with exact statements about how much evidence a fault can ever
generate, at what rate, with what fluctuation around it, and which faults can
generate none at all, for any detector.

Our contributions are the following.

\begin{enumerate}[leftmargin=1.6em, itemsep=0.15em]
\item \textbf{The guarantee as accounting.} For the online quantile tracker,
the false-alarm guarantee is a telescoping identity on the internal state,
pathwise and insensitive to any alternative (Proposition~\ref{prop:validity}).
\item \textbf{Exact evidence identities, with exact boundaries.} A persistent
step of height $\delta$ produces excess alarm mass within one alarm of
$\delta/\eta$ in general, and exactly $\delta/\eta$ on every trajectory when
$\delta\in\eta\mathbb Z$, after an exact-coupling time of bounded expectation
(Theorem~\ref{thm:budget}); a ramp of slope $c$ produces a stationary excess
rate of exactly $c/\eta$ for $c<\eta(1-\alpha)$, and the condition is
necessary and sufficient: the identity value reaches the alarm cap exactly at
the boundary, beyond which the rate is $1$
(Theorem~\ref{thm:ramp}, Proposition~\ref{prop:sat}).
\item \textbf{An exact fluctuation law for the certificate.} The windowed
alarm rate telescopes to a difference of the tight internal state, so its
null standard deviation scales as $1/L$ and not the binomial $1/\sqrt L$;
the constant is closed-form and validated at $0.2$--$2.7\%$ over twelve
configurations (Theorem~\ref{thm:fluct}). Binomial calibration of downstream
detectors is therefore miscalibrated by exactly $\sqrt{\eta\varphi(q_0)L}$,
and correct calibration shortens the detection window of a slope-$c$ drift
from $\Theta((\eta/c)^2)$ to $\Theta(\eta/c)$ (Corollary~\ref{cor:window}),
a factor confirmed head-to-head in experiment.
\item \textbf{The manufactured blind set.} Any monitor with uniform size
$\alpha$ over a drift class $\mathcal D$ has power at most $\alpha$ against
every fault in $\mathcal D-\mathcal D$, for any decision rule and horizon
(Theorem~\ref{thm:confusion}); the proof is an elementary two-point
identification, stated as such, and the contribution is the object: for
speed-bounded classes the blind set is exactly the doubled-speed class
(Corollary~\ref{cor:speed}), and an exact Gaussian projection bound, never
vacuous, replaces Pinsker outside it (Proposition~\ref{prop:lecam}).
\item \textbf{The self-manufacturing link.} The tracker's own gain determines
the drift class it effectively tolerates, so the blind set is indexed by a
design parameter rather than by any physical uncertainty model
(Proposition~\ref{prop:manufacture}); tuning the gain trades the step budget
against the width of the blind set, and no tuning removes both.
\end{enumerate}

Section~\ref{sec:setting} fixes the setting. Section~\ref{sec:identities}
proves the accounting, evidence and fluctuation results.
Section~\ref{sec:blind} develops the blind set. Section~\ref{sec:exp} reports
the experiments, each against its stated precision.
Sections~\ref{sec:related} and~\ref{sec:discussion} position and discuss.

\section{Setting}\label{sec:setting}

A monitor observes a real-valued score stream $(s_t)_{t \ge 1}$ and maintains
a threshold $q_t$. At each step it emits an alarm
$e_t = \mathbf{1}\{s_t > q_t\}$ and updates
\begin{equation}\label{eq:tracker}
q_{t+1} \;=\; q_t + \eta\,(e_t - \alpha),
\end{equation}
with adaptation gain $\eta > 0$ and target level $\alpha \in (0,1)$. An alarm
raises the threshold by $\eta(1-\alpha)$; a quiet step lowers it by
$\eta\alpha$. Update~\eqref{eq:tracker} is stochastic gradient descent on the
pinball loss of quantile regression \cite{koenker1978} and is the
direct-threshold analogue of the adaptive conformal
update of \cite{gibbs2021}, which applies the same recursion to a confidence
level; the identities below hold for that variant to first order in $\eta$
(Remark~\ref{rem:scope}). We write the cumulative excess alarm mass over a window as
$E_{[t_1, t_2]} = \sum_{t = t_1}^{t_2} (e_t - \alpha)$.

Faults and drifts act additively on the score stream:
$s_t = \xi_t + \mu_t + \theta_t$, where $(\xi_t)$ is a stationary noise
sequence with distribution function $F$, $\mu$ is a nuisance drift the
monitor is required to tolerate, and $\theta$ is the fault. All magnitudes
are in score units. When the noise is standard Gaussian we write
$q_0=\Phi^{-1}(1-\alpha)$ and $\varphi(q_0)$ for the density at the tracked
quantile.

\section{Accounting, evidence and fluctuation}\label{sec:identities}

\begin{lemma}[telescoping]\label{lem:tel}
For any score sequence and any $t_1 \le t_2$,
$E_{[t_1,t_2]} = (q_{t_2+1} - q_{t_1})/\eta$.
\end{lemma}
\begin{proof}
Sum~\eqref{eq:tracker} over the window.
\end{proof}

\begin{proposition}[validity is accounting; cf.\ \cite{gibbs2021}, Prop.~4.1]\label{prop:validity}
For any, possibly adversarial, score sequence such that the threshold remains
in a bounded set of diameter $D_q$,
$\big|\tfrac1T\sum_{t=1}^{T} e_t - \alpha\big| \le D_q/(\eta T)$.
This deterministic bound is the false-alarm guarantee of adaptive conformal
inference \cite{gibbs2021}, restated for the direct-threshold tracker; we
record it because every result below is drawn from the same telescoping
identity, read pathwise rather than as a bound.
In particular, whenever the fault leaves the threshold bounded, as every
fault of bounded cumulative height does, the long-run alarm rate equals
$\alpha$ regardless of the fault's presence, size or shape; unbounded drifts
move the threshold, not the identity, and their exact rates are computed in
Theorem~\ref{thm:ramp}.
\end{proposition}

The guarantee is thus a property of the update rule, not of the world: it
certifies that the bookkeeping closes, not that anything was watched. The
next results make the cost exact. We compare a baseline world and a faulted
world coupled on the same noise, and measure the excess alarm mass the fault
adds.

\begin{lemma}[pathwise coupling]\label{lem:coupling}
Let two trackers obey~\eqref{eq:tracker} on the same score stream
$(\xi_t)_{t\ge t_0}$ with initial thresholds $u_{t_0}\le v_{t_0}$, and set
$g_t=v_t-u_t$. Then:
\emph{(i)} $g_t\in g_{t_0}+\eta\mathbb Z$ for all $t$;
\emph{(ii)} $g$ is nonincreasing while $g_t\ge\eta$, decreasing by exactly
$\eta$ at every step with $\xi_t\in(u_t,v_t]$, and from the first time
$T_1$ with $g_{T_1}<\eta$ onward, $|g_t|<\eta$ forever;
\emph{(iii)} if moreover $g_{t_0}\in\eta\mathbb Z$, then $g_t=0$ for all
$t\ge T_1$: the trackers couple exactly and are identical thereafter;
\emph{(iv)} if the $\xi_t$ are i.i.d.\ with a density positive on an interval
containing the recurrent range of the pair, then $T_1<\infty$ almost surely
and $\mathbb E[T_1]\le C\,\lceil g_{t_0}/\eta\rceil$ for a constant $C$
depending only on the noise law, $\eta$ and $\alpha$
(Appendix~\ref{app:ramp}).
\end{lemma}
\begin{proof}
If $\xi_t>\max(u_t,v_t)$ or $\xi_t\le\min(u_t,v_t)$, both trackers move
identically ($e_t$ agrees) and $g$ is unchanged. Otherwise $\xi_t$ separates
them, and exactly one alarms: whichever tracker has the \emph{smaller}
threshold sees $\xi_t$ exceed it and rises by $\eta(1-\alpha)$, while the
other, with the larger threshold, stays quiet and falls by $\eta\alpha$.
While $g_t=v_t-u_t\ge\eta$, the smaller threshold is $u_t$: $u$ rises,
$v$ falls, so $g$ decreases by exactly $\eta$; this is monotonicity, and
increments lie in $\{0,-\eta\}$ in this regime. The first crossing lands
$g_{T_1}\in(-\eta,\eta)$. The invariance of this interval is where the role
of ``smaller threshold'' matters: if $g_t\in(0,\eta)$, $u_t$ is still smaller
and the same shrink event decreases $g$ by $\eta$, landing in
$(-\eta,0)$; but if $g_t\in(-\eta,0)$, the ordering has flipped, $v_t$ is
now the smaller threshold, so $v$ (not $u$) rises and $u$ falls on a shrink
event, which \emph{increases} $g=v-u$ by $\eta$, landing back in $(0,\eta)$.
The two sub-cases are mirror images once the smaller threshold is correctly
identified, and together they show $(-\eta,\eta)$ is invariant under every
shrink event from either side, which is (i)--(ii). Under (iii) the lattice
(i) forces $g_{T_1}=0$, and two trackers with equal thresholds on the same
stream remain equal. Claim (iv) is the hitting-time bound proved in the
appendix: the gap is bounded and nonincreasing, so the pair lives in a
strip around the diagonal and joint returns reduce to returns of a single
tracker, which is positively recurrent by drift; while
$g_t\ge\eta$ the shrink probability is $F(v_t)-F(u_t)\ge F(u_t+\eta)-F(u_t)$,
bounded below on compacts, so
each of the at most $\lceil g_{t_0}/\eta\rceil$ shrinks occurs after a wait
of uniformly bounded expectation.
\end{proof}

\begin{theorem}[step budget, pathwise]\label{thm:budget}
Couple two trackers on the same i.i.d.\ noise with full-support density, one
observing $s_t=\xi_t$, the other $s^f_t=\xi_t+\delta\,\mathbf 1\{t\ge t_0\}$
with $\delta>0$ persistent, both with the same threshold at $t_0$, lying in
a fixed compact set. Then there
is an almost surely finite time $T^*$ with
$\mathbb E[T^*-t_0]\le C\lceil\delta/\eta\rceil$ such that for \emph{all}
$T\ge T^*$,
\[
\Big|\,E^{f}_{[t_0,T]}-E_{[t_0,T]}-\frac{\delta}{\eta}\,\Big|\;<\;1,
\]
and if $\delta\in\eta\mathbb Z$ the difference equals $\delta/\eta$
\emph{exactly, on every trajectory}, for all $T\ge T^*$. The budget does not
depend on the noise law, on $\alpha$, or on how long the fault persists, and
no moment condition is required: heavy tails change nothing.
\end{theorem}
\begin{proof}
By Lemma~\ref{lem:tel} applied to each world,
$E^{f}_{[t_0,T]}-E_{[t_0,T]}=(q^{f}_{T+1}-q_{T+1})/\eta$. The shifted process
$u_t=q^f_t-\delta$ obeys~\eqref{eq:tracker} on the same noise, since
$\mathbf 1\{\xi_t+\delta>q^f_t\}=\mathbf 1\{\xi_t>u_t\}$; thus $u$ and $q$
are two copies started $\delta$ apart on the same stream, and
Lemma~\ref{lem:coupling}(ii)--(iv) give $|q^f_t-q_t-\delta|<\eta$ for all
$t\ge T^*=T_1$, with equality to $0$ in the lattice case (iii). Only
continuity and full support of $F$ enter, through the shrink probability.
\end{proof}

After $T^*$ the faulted monitor's alarm rate returns to $\alpha$: the monitor
is re-blinded, holding a threshold exactly $\delta$ above the baseline
world's, and the books are closed at $\delta/\eta$.

\begin{theorem}[ramp excess rate]\label{thm:ramp}
Let the faulted stream be $s^f_t=\xi_t+c\,t$ with slope
$0\le c<\eta(1-\alpha)$, the noise i.i.d.\ with full-support density. Then,
almost surely,
$\tfrac1T\sum_{t=1}^{T}e_t\to\alpha+c/\eta$,
and in any stationary regime of the recentred chain the alarm probability
equals $\alpha+c/\eta$ exactly, independently of the accumulated magnitude
$c\,t$.
\end{theorem}
\begin{proof}
Set $p_t=q_t-ct$. Then $e_t=\mathbf 1\{\xi_t>p_t\}$ and
$p_{t+1}=p_t+\eta(e_t-\alpha)-c$: the recentred chain is a tracker on the
stationary noise with constant leak $c$. Telescoping gives
$\tfrac1T\sum_t e_t=\alpha+\tfrac c\eta+\tfrac{p_{T+1}-p_1}{\eta T}$, and the
remainder vanishes almost surely because $(p_t)$ has bounded increments and
strictly inward drift outside a compact interval whenever
$c<\eta(1-\alpha)$, hence $\sup_t|p_t|<\infty$ a.s.\
(Appendix~\ref{app:ramp}). Stationarity forces
$\eta(\mathbb E[e_t]-\alpha)=c$.
\end{proof}

\begin{proposition}[saturation: the tracking condition is the alarm cap]\label{prop:sat}
The condition $c<\eta(1-\alpha)$ is necessary and sufficient, and the
identity is self-limiting: its value $\alpha+c/\eta$ equals exactly $1$ at
$c=\eta(1-\alpha)$. For $c\ge\eta(1-\alpha)$ the empirical alarm rate
converges to $1$ almost surely; at the boundary the recentred chain is
nonincreasing and drifts to $-\infty$ through waits of exploding expectation,
and above it the divergence is linear (Appendix~\ref{app:ramp}).
\end{proposition}

Theorems~\ref{thm:budget} and~\ref{thm:ramp} give the complete dichotomy for
piecewise-affine faults. A fault that ends at total height $B$ disburses total
evidence mass $B/\eta$ whatever its path (Lemma~\ref{lem:tel} end to
end); the path only decides the rate of disbursement. A fast fault
concentrates its budget and clears the noise floor; a slow fault disburses at
rate $c/\eta$ and can remain below any detection resolution indefinitely.
What that resolution actually is, the next result computes, and the answer
carries an immediate operational consequence: because the certificate's null
fluctuation obeys a $1/L$ law rather than the binomial $1/\sqrt L$, a detector
calibrated on the binomial scale, which is the scale of the Hoeffding term in
the concentration bound of \cite{gibbs2021}, overstates
the null noise by a factor growing as $\sqrt{\eta\varphi(q_0)L}$ and needs a
\emph{quadratic} rather than linear observation window in the inverse fault
speed. Corollary~\ref{cor:window} makes that comparison exact; the theorem
first establishes why the law holds.

\begin{theorem}[exact fluctuation law of the certificate]\label{thm:fluct}
\emph{(i) Pathwise structure.} For any stream, the windowed alarm mass
telescopes: $\sum_{s=t}^{t+L-1}(e_s-\alpha)=(q_{t+L}-q_t)/\eta$. Hence in
any stationary null regime the windowed alarm rate
$\bar e_{t,L}=\tfrac1L\sum_{s=t}^{t+L-1}e_s$ satisfies
\[
\operatorname{Var}(\bar e_{t,L})
=\frac{\operatorname{Var}(q_{t+L}-q_t)}{\eta^2L^2}
\;\le\;\frac{4\operatorname{Var}_\pi(q)}{\eta^2L^2}:
\]
the null fluctuation of the certificate scales as $1/L$, not the binomial
$1/\sqrt L$, because it is a difference of a tight internal state.
\emph{(ii) Constants.} For i.i.d.\ noise with density $\varphi(q_0)>0$ at the
tracked quantile, the linearized (small-$\eta$) stationary chain is AR(1)
with autocorrelation $\rho=1-\eta\varphi(q_0)$ and variance
$V_q=\eta\,\alpha(1-\alpha)/(2\varphi(q_0))+O(\eta^2)$
\cite{kushneryin2003}, giving the closed form
\[
\sigma_L\;=\;\frac{\sqrt{2V_q\,(1-\rho^{\,L})}}{\eta L}
\;\xrightarrow[L\gg1/(\eta\varphi(q_0))]{}\;
\frac{1}{L}\sqrt{\frac{\alpha(1-\alpha)}{\eta\,\varphi(q_0)}}\,.
\]
\emph{(iii) Miscalibration ratio.} A detector calibrated on the binomial
scale $\sqrt{\alpha(1-\alpha)/L}$ overstates the null fluctuation by the
factor
$\sqrt{\alpha(1-\alpha)/L}\,/\,\sigma_L\to\sqrt{\eta\,\varphi(q_0)\,L}$,
and is conservative and blind in exactly that proportion.
\end{theorem}
\begin{proof}
(i) is Lemma~\ref{lem:tel} plus stationarity;
$\operatorname{Var}(q_{t+L}-q_t)=2V_q(1-\operatorname{corr}(q_t,q_{t+L}))
\le 4V_q$ by Cauchy--Schwarz, and $\le 2V_q$ whenever the lag-$L$
correlation is nonnegative, as it is in the linearized regime of (ii).
For (ii), linearize $e_t$ around the stationary point:
$q_{t+1}-q_0\approx(1-\eta\varphi(q_0))(q_t-q_0)+\eta\,m_t$ with $m_t$ a
bounded martingale difference of variance $\alpha(1-\alpha)$; the stationary
AR(1) variance is $\eta^2\alpha(1-\alpha)/(1-\rho^2)=V_q(1+O(\eta))$, its
lag-$L$ correlation is $\rho^L$, and substitution in (i) gives the display.
(iii) is the ratio of the two expressions.
\end{proof}

The constants of (ii) are the classical stationary behaviour of
constant-step stochastic approximation \cite{kushneryin2003}; what is new is
the pathwise identity (i) that ties the certificate's fluctuation to the
state, the resulting $1/L$ law for the windowed rate, and its consequence
(iii) for every detector deployed downstream of a self-calibrating monitor.
Validation with no fitted parameter: over twelve configurations
(two gains, six window lengths) the measured standard deviation matches the
closed form at $0.2$--$2.7\%$, and the measured miscalibration ratio at
$L=3200$ is $5.23$ against a predicted $5.30$
(Table~\ref{tab:fluct}, Figure~\ref{fig:fluct}).

\begin{corollary}[two detection regimes]\label{cor:window}
Consider a downstream detector flagging when a window of length $L$ exceeds
its calibrated null band by $z$ standard units, against a drift of excess
rate $c/\eta$.
\emph{(i) Correct calibration} (band $\alpha+z\,\sigma_L$): the minimal
resolvable window is
\[
L^*\;=\;\frac{z}{c}\,\sqrt{\frac{\alpha(1-\alpha)\,\eta}{\varphi(q_0)}},
\]
linear in the inverse fault speed: halving the fault speed \emph{doubles}
the observation needed.
\emph{(ii) Binomial calibration} (band
$\alpha+z\sqrt{\alpha(1-\alpha)/L}$, the scale of the Hoeffding term in
\cite{gibbs2021}, Theorem~4.1): the minimal
window is $L^*=z^2\alpha(1-\alpha)(\eta/c)^2$, quadratic in the inverse
speed: halving the speed quadruples it. The quadratic law is real but it is
the price of miscalibration, not of information.
\end{corollary}

\begin{remark}[assumption load]
Proposition~\ref{prop:validity} is pathwise and assumption-free;
Theorem~\ref{thm:confusion} below requires only equality of the noise law
across worlds; the i.i.d.\ and full-support assumptions carry the exact
budgets through the coupling; the fluctuation constants of
Theorem~\ref{thm:fluct}(ii) alone are linearized, and their measured accuracy
is reported next to every use. The strongest statements are the least loaded.
\end{remark}

\begin{remark}[scope]\label{rem:scope}
The identities are exact for the linear-gain tracker~\eqref{eq:tracker}.
Sliding-window quantile variants and the level-update form of adaptive
conformal inference obey them to first order; the experiments use
\eqref{eq:tracker} and measure agreement at the stated precision.
\end{remark}

\section{The manufactured blind set}\label{sec:blind}

The identities above bound what one specific monitor family extracts. The
next result bounds what any procedure can extract, once a tolerance
obligation is imposed.

Fix a horizon $n$ and the observation model
$y_t=\mu_t+\theta_t+\xi_t$, $t=1,\dots,n$, with $(\xi_t)$ i.i.d.\ noise with
full-support density, nuisance drift $\mu\in\mathcal D$, and fault trajectory
$\theta$. The tolerance obligation is the requirement that a decision rule
$\phi$ (any measurable function of the observations, adaptive or not)
satisfies the uniform size constraint
$\sup_{\mu\in\mathcal D}\Pr_{\mu,\theta=0}(\phi=1)\le\alpha$.

\begin{theorem}[confusion]\label{thm:confusion}
Let $\theta\in\mathcal D-\mathcal D$, that is $\theta=\mu'-\mu''$ with
$\mu',\mu''\in\mathcal D$. Then every rule $\phi$ satisfying the uniform size
constraint obeys $\Pr_{\mu'',\theta}(\phi=1)\le\alpha$. Power is at most
size: the fault is invisible to any detector, at any horizon.
\end{theorem}
\begin{proof}
Under $(\mu'',\theta)$ the observations are
$y_t=\mu''_t+\theta_t+\xi_t=\mu'_t+\xi_t$, whose law is that of the null
instance $(\mu',0)$ with $\mu'\in\mathcal D$. The size constraint applied to
this instance gives the bound.
\end{proof}

The proof is a two-point identification argument in the sense of Le Cam,
elementary by design \cite{tsybakov2009}; we state this plainly. The
contribution is not the technique but the object it isolates: the difference
set $\mathcal D-\mathcal D$ as the exact blind set induced by a tolerance
obligation, its closed-form characterization for speed classes below, and
the fact, established in Proposition~\ref{prop:manufacture}, that the
tracker absorbs a speed class fixed by its own gain, so that under the
certification reading of Corollary~\ref{cor:certif} it manufactures
$\mathcal D$. The classical
geometric theory of undetectable faults, those lying in the
disturbance-decoupled subspace \cite{massoumnia1989,chenpatton1999}, is the
linear, offline, model-fixed instance of the same mechanism; here the class
is nonparametric, the rule arbitrary and adaptive, and the set is created by
the monitor's design rather than by the plant.

The constraint bears on size only. Anytime-valid procedures whose type~I
error is controlled uniformly over the null class, such as tests built from
e-processes or conformal test martingales \cite{vovk2005,ramdas2023},
satisfy the hypothesis at every stopping time and are covered verbatim: no
betting scheme recovers what the tolerance obligation removes.

\begin{corollary}[speed frontier]\label{cor:speed}
Let $\mathcal D_v=\{\mu:\mu_1=0,\ |\mu_{t+1}-\mu_t|\le v\}$. Then every fault
with $\theta_1=0$ and $|\theta_{t+1}-\theta_t|\le 2v$ lies in
$\mathcal D_v-\mathcal D_v$ and is invisible to any monitor with uniform
size over $\mathcal D_v$, whatever its final magnitude. Conversely the
difference of two $v$-speed-bounded drifts is $2v$-speed-bounded, so
$\mathcal D_v-\mathcal D_v$ is exactly the $2v$-speed class: the blind set is
characterized, not merely contained.
\end{corollary}
\begin{proof}
Write $\theta=(\theta/2)-(-\theta/2)$; both halves are $v$-speed-bounded,
start at zero, and $\mathcal D_v$ is symmetric.
\end{proof}

What decides invisibility is the speed of the fault relative to the tolerated
speed, not its size: the frontier sits at exactly twice the tolerance. We
claim no achievability beyond the set: the projection bound below caps the
power of every rule outside it, and whether some rule attains nontrivial
power arbitrarily close to the frontier is left open. Two frontiers must not
be conflated: the information-theoretic frontier at $2v$, which is exact and
detector-free (this corollary), and the operational frontier of any
\emph{specific} downstream detector, which is the resolution crossing of
Corollary~\ref{cor:window} and has a finite width; the experiments exhibit
both, separately.

\begin{proposition}[exact projection bound]\label{prop:lecam}
Let the noise be Gaussian, $\xi_t\sim\mathcal N(0,\sigma^2)$, and let
$\Pi_{\mathcal D}$ denote Euclidean projection onto $\mathcal D$ over the
horizon. For every rule with uniform size $\alpha$ over $\mathcal D$ and
every $(\mu,\theta)$ with $\mu\in\mathcal D$, writing
$\Delta=(\mu+\theta)-\Pi_{\mathcal D}(\mu+\theta)$,
\[
\Pr_{\mu,\theta}(\phi=1)\;\le\;\alpha\;+\;
2\,\Phi\!\Big(\frac{\|\Delta\|_2}{2\sigma}\Big)-1 .
\]
\end{proposition}
\begin{proof}
Choose the null instance $\mu^\ast=\Pi_{\mathcal D}(\mu+\theta)$. Power
exceeds size at $\mu^\ast$ by at most the total variation distance between
the two observation laws, which are Gaussian with equal covariance
$\sigma^2 I_n$ and mean gap $\Delta$; for such a pair the total variation is
exactly $2\Phi(\|\Delta\|_2/(2\sigma))-1$, by reduction to the
one-dimensional likelihood-ratio direction.
\end{proof}

\begin{remark}[on Pinsker]
The Pinsker route, $\mathrm{TV}\le\sqrt{\mathrm{KL}/2}$ with
$\mathrm{KL}=\|\Delta\|_2^2/(2\sigma^2)$, gives the linear bound
$\|\Delta\|_2/(2\sigma)$: weaker than the exact bound by the factor
$\sqrt{2/\pi}\approx0.80$ in the near-class regime, and vacuous for
$\|\Delta\|_2>2\sigma$, where it exceeds $1$. The exact bound saturates at
$1$, as any upper bound on power must once detection becomes easy, and is
informative at every distance.
\end{remark}

The distance of the perturbed trajectory to the tolerated class is the whole
story: the blind set is its zero level set, and power decays continuously as
the fault approaches the class.

\begin{proposition}[the monitor manufactures its own class]\label{prop:manufacture}
Say a drift $\nu$ is $\varepsilon$-absorbed by a monitor if, almost surely,
the long-run alarm rate under $(\nu,\theta=0)$ deviates from $\alpha$ by at
most $\varepsilon$. For every $\varepsilon<\min(\alpha,1-\alpha)$, the
tracker~\eqref{eq:tracker} $\varepsilon$-absorbs \emph{every} drift of the
speed class $\mathcal D_{v_{\mathrm{eff}}}$ with
$v_{\mathrm{eff}}=\varepsilon\,\eta$: when
$\sup_t|\nu_{t+1}-\nu_t|\le\varepsilon\eta$, the recentred chain
$q_t-\nu_t$ keeps bounded increments and uniformly inward drift, so it is
tight and the long-run alarm rate lies in
$[\alpha-\varepsilon,\alpha+\varepsilon]$; constant slopes give the exact
value $\alpha+c/\eta$ of Theorem~\ref{thm:ramp}
(Appendix~\ref{app:ramp}). The tolerated
class is indexed by the monitor's own gain, not by any physical uncertainty
model.
\end{proposition}

\begin{corollary}[certification reading]\label{cor:certif}
Suppose a certification regime declares normal every drift that the deployed
monitor absorbs at its operating resolution, and requires any downstream
decision rule to keep uniform size over the declared normal class. Then every
fault with speed at most $2\varepsilon\eta$ is invisible to every compliant
rule, whatever its final magnitude. The blind set is manufactured by the
monitor and inherited by the whole certification chain.
\end{corollary}

\begin{remark}[no free tuning]
Raising $\eta$ shrinks nothing for free: it widens the tolerated class, hence
the blind set, in exact proportion to how fast it re-absorbs steps. The
operating resolution is now exact rather than heuristic: at correct
calibration $\varepsilon(L)=z\sqrt{\alpha(1-\alpha)/(\eta\varphi(q_0))}/L$,
at binomial calibration $\varepsilon(L)=z\sqrt{\alpha(1-\alpha)/L}$
(Theorem~\ref{thm:fluct}), so slowing adaptation narrows the blind set only
at the price of longer miscalibration after legitimate changes. The dial
trades the evidence budget against the width of the blind set; no setting
removes both costs.
\end{remark}

\begin{remark}[worst case versus average case]
Theorem~\ref{thm:confusion} is a worst-case statement over the nuisance
class: the alternative world $\mu'$ must merely be admissible. If a prior
over $\mathcal D$ concentrates away from $\mu'$, average-case detection may
recover; the geometry then enters through the distance of
Proposition~\ref{prop:lecam} weighted by the prior. The worst-case reading
is the relevant one for certification, where guarantees must hold for every
admissible drift.
\end{remark}

\section{Numerical verification}\label{sec:exp}

All experiments use the tracker~\eqref{eq:tracker} with $\alpha=0.1$,
$\eta=0.05$, standard Gaussian noise, paired noise realizations between
baseline and faulted worlds, and fixed seeds per repetition. Every number is
an evaluation of recursion~\eqref{eq:tracker} itself, the object under
study; each prediction below was computed before the corresponding
measurement, and none contains a fitted parameter. Code reproducing every
table and figure is available from the author.

\textbf{Identities, with per-trajectory exactness.} Table~\ref{tab:ident}
reports the step budgets over $50$ repetitions with $95\%$ confidence
intervals, together with the maximal deviation from $\delta/\eta$ over all
individual trajectories. All listed $\delta$ are lattice multiples of
$\eta=0.05$, and the measured maximal per-trajectory deviation is $0.000$:
the identity holds \emph{on every path}, as Theorem~\ref{thm:budget}
predicts, not merely on average. For the non-lattice step $\delta=0.53$
(predicted budget $10.6$), the measured excess over $50$ trajectories has
mean $10.52$, $95\%$ interval half-width $0.14$, and per-trajectory range
$[10.00,\,11.00]$: the sub-alarm oscillation of
Lemma~\ref{lem:coupling}(ii), confined to the predicted unit window. Ramp
rates over $20$ repetitions match $\alpha+c/\eta$ with confidence intervals
below $10^{-4}$.

\begin{table}[h]
\centering\small
\begin{tabular}{lcccccc}
\toprule
step $\delta$ & 0.5 & 1.0 & 1.5 & 2.0 & 2.5 & 3.0\\
predicted $\delta/\eta$ & 10 & 20 & 30 & 40 & 50 & 60\\
measured mean (CI$_{95}$) & 10.00 (0.00) & 20.00 (0.00) & 30.00 (0.00) &
40.00 (0.00) & 50.00 (0.00) & 60.00 (0.00)\\
max per-path $|$dev$|$ & 0.000 & 0.000 & 0.000 & 0.000 & 0.000 & 0.000\\
\midrule
ramp slope $c$ ($\times10^{-3}$) & 0.5 & 1 & 2 & 4 & 8 & \\
predicted $\alpha+c/\eta$ & 0.110 & 0.120 & 0.140 & 0.180 & 0.260 & \\
measured rate & 0.1100 & 0.1200 & 0.1400 & 0.1800 & 0.2600 & \\
\bottomrule
\end{tabular}
\caption{Exact evidence identities. Steps: 50 repetitions, horizon 4000;
lattice steps are exact on every trajectory. Ramps: 20 repetitions,
horizon 40000 after burn-in; all CI$_{95}$ below $10^{-4}$.}
\label{tab:ident}
\end{table}

\textbf{Saturation.} Sweeping the slope through the boundary
$\eta(1-\alpha)=0.045$ (four repetitions, horizon $10^5$): measured rates
$0.5000$, $0.7000$, $0.9000$, $0.9900$, $0.9999$, $1.0000$ for
$c=0.02$, $0.03$, $0.04$, $0.0445$, $0.045$, $0.05$, against
$\min(\alpha+c/\eta,1)$. The identity runs exactly into its own cap at the
boundary, as Proposition~\ref{prop:sat} states.

\textbf{Fluctuation law.} Table~\ref{tab:fluct} compares the closed form of
Theorem~\ref{thm:fluct} with the measured standard deviation of the windowed
null alarm rate, across two gains and six window lengths (six repetitions of
horizon $3\times10^5$ each). Mean error $1.0\%$, maximum $2.7\%$, no fitted
parameter; the stationary threshold deviation itself matches
($\eta=0.05$: predicted $0.1132$, measured $0.1140$; $\eta=0.02$: $0.0716$,
$0.0715$), and the miscalibration ratio at $L=3200$ is measured at $5.23$
and $3.26$ against predicted $\sqrt{\eta\varphi(q_0)L}=5.30$ and $3.35$.

\begin{table}[h]
\centering\small
\begin{tabular}{r|ccc|ccc}
\toprule
& \multicolumn{3}{c|}{$\eta=0.05$} & \multicolumn{3}{c}{$\eta=0.02$}\\
$L$ & predicted & measured & binom./true & predicted & measured & binom./true\\
\midrule
100 & 2.451e-2 & 2.478e-2 & 1.21 & 2.757e-2 & 2.771e-2 & 1.08\\
200 & 1.457e-2 & 1.460e-2 & 1.45 & 1.799e-2 & 1.792e-2 & 1.18\\
400 & 7.888e-3 & 7.942e-3 & 1.89 & 1.100e-2 & 1.102e-2 & 1.36\\
800 & 4.001e-3 & 3.977e-3 & 2.67 & 6.137e-3 & 6.237e-3 & 1.70\\
1600 & 2.002e-3 & 1.966e-3 & 3.82 & 3.159e-3 & 3.185e-3 & 2.35\\
3200 & 1.001e-3 & 1.014e-3 & 5.23 & 1.582e-3 & 1.626e-3 & 3.26\\
\bottomrule
\end{tabular}
\caption{Null standard deviation of the windowed alarm rate: closed form
versus measurement, and the measured binomial miscalibration ratio.}
\label{tab:fluct}
\end{table}

\begin{figure}[t]\centering
\includegraphics[width=0.62\textwidth]{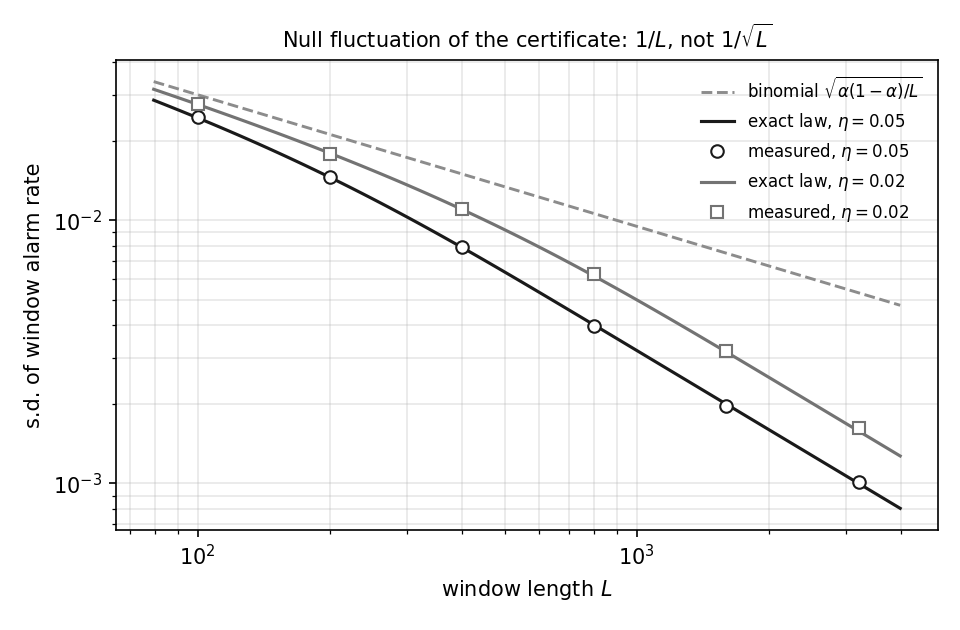}
\caption{The certificate's null fluctuation follows the $1/L$ law of
Theorem~\ref{thm:fluct} (solid, no fitted parameter), not the binomial
$1/\sqrt L$ (dashed); measured points for two gains.}
\label{fig:fluct}
\end{figure}

\textbf{The two calibrations, head to head.} At $L=800$, the exact null
scale is $\sigma_{800}=4.00\times10^{-3}$ against a binomial
$1.06\times10^{-2}$, ratio $2.65$, equal to the predicted
$\sqrt{\eta\varphi(q_0)L}=2.65$. Two detectors were run on identical
streams (40 repetitions, horizon $6\times10^4$): the binomial-calibrated
band at $z=3$ (threshold $0.1318$) and the exactly calibrated band at $z=5$
(threshold $0.1200$), the latter chosen so that its measured false-flag rate
per horizon over sliding windows equals the former's ($0.00$, Wilson
$[0,0.09]$, for both). At matched null behaviour, power against sustained
drifts of excess rate $c/\eta\in\{0.012,0.016,0.020,0.032\}$ is
$1.00/1.00/1.00/1.00$ for the exact calibration against
$0.00/0.12/0.97/1.00$ for the binomial one: correct calibration resolves
drifts $2$--$3$ times slower at identical false-alarm behaviour, the factor
predicted by Theorem~\ref{thm:fluct}(iii).

\textbf{Illustrative trace.} Figure~\ref{fig:trace} shows a step fault of size
$4\sigma$ producing its evidence burst (peak windowed rate $0.26$) and then
re-blinding the monitor, next to a ramp fault reaching $8\sigma$, twice the
size, whose windowed alarm rate never leaves the neighbourhood of $\alpha$
(maximum $0.14$ over the run) while the threshold rides the fault. The
certificate remains clean because the fault is invisible, not despite it.
This is the simulated counterpart of the collapse-before-awareness mode
reported for world-model-based monitors of reinforcement learning agents
under gradual drift \cite{gradual}, discussed in Section~\ref{sec:related}:
the same absorption, here with the threshold trajectory exposed.

\begin{figure}[h]
\centering
\includegraphics[width=0.8\textwidth]{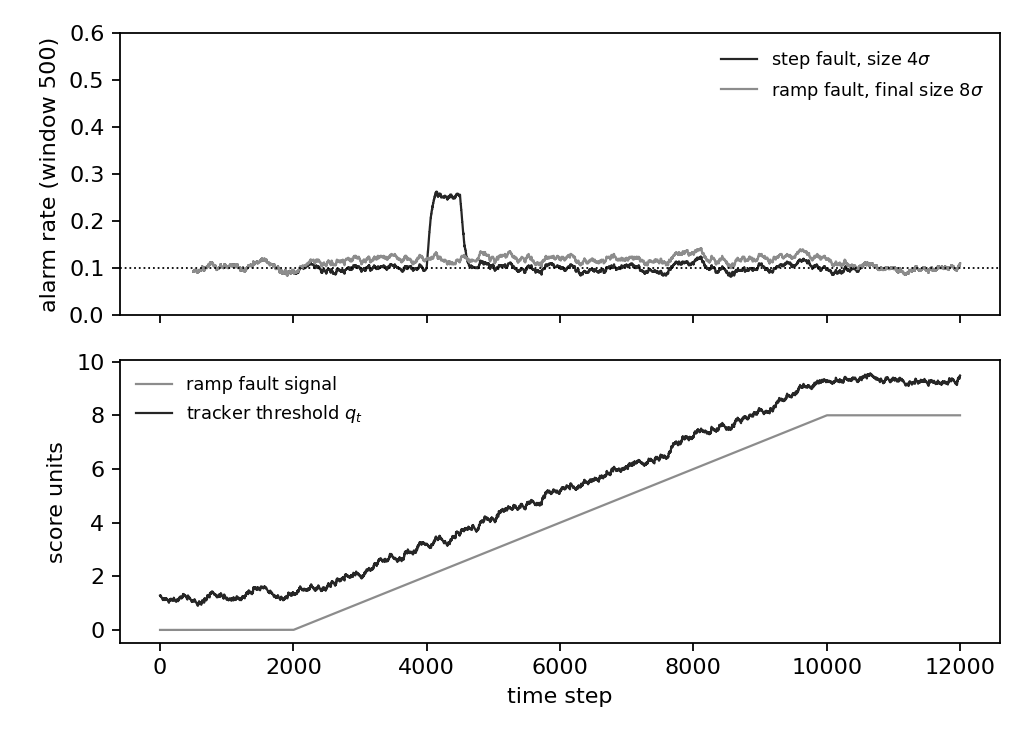}
\caption{Top: windowed alarm rate for a $4\sigma$ step and an $8\sigma$
ramp; the dotted line is the target $\alpha$. Bottom: the tracker threshold
riding the ramp fault.}
\label{fig:trace}
\end{figure}

\textbf{Speed frontier, with intervals.} Table~\ref{tab:front} refines the
power map of Figure~\ref{fig:map} along the speed axis at two fault sizes
(40 repetitions per cell, Wilson $95\%$ intervals; ramp to size $B$ then
hold; binomial-calibrated downstream detector, $L=400$, $z=3$, resolution
$0.045$, measured null flag rate $0.00$ on matched fault-free horizons). The
transition is a resolution crossing completed within
$c/\eta\in[0.016,0.035]$, not a discretization artifact, and it is vertical
to first order: columns decide. Within the transition band, size matters
through duration: at $c/\eta=0.025$ the $16\sigma$ fault is caught in $0.97$
of runs against $0.40$ for the $4\sigma$ fault, because the longer ramp
offers the detector more windows, the repeated-attempt mechanism. Outside
the band the rows are flat: at $0.016$ even $16\sigma$ faults are mostly
missed; at $0.035$ and above even $4\sigma$ faults are always caught.

\begin{table}[h]
\centering\small
\begin{tabular}{lcc}
\toprule
$c/\eta$ & power, $B=4\sigma$ & power, $B=16\sigma$\\
\midrule
0.016 & 0.03 [0.00, 0.13] & 0.12 [0.05, 0.26]\\
0.025 & 0.40 [0.26, 0.55] & 0.97 [0.87, 1.00]\\
0.035 & 1.00 [0.91, 1.00] & 1.00 [0.91, 1.00]\\
0.045 & 1.00 [0.91, 1.00] & 1.00 [0.91, 1.00]\\
0.055--0.100 & 1.00 [0.91, 1.00] & 1.00 [0.91, 1.00]\\
\bottomrule
\end{tabular}
\caption{Refined speed frontier with Wilson intervals, 40 repetitions per
cell; same protocol as Figure~\ref{fig:map}.}
\label{tab:front}
\end{table}

\begin{figure}[h]
\centering
\includegraphics[width=0.62\textwidth]{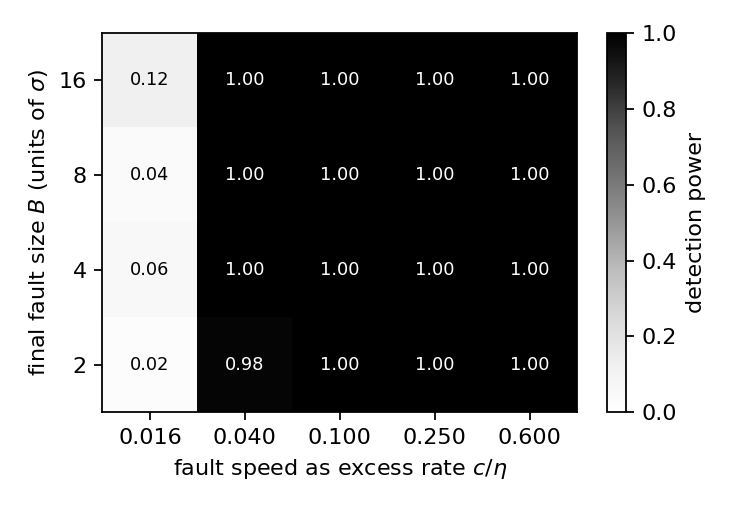}
\caption{Detection power versus fault speed (as excess rate $c/\eta$) and
final fault size. Protocol: $2000$ fault-free steps, a ramp of slope $c$ to
height $B$ (duration $B/c$), then $3000$ steps at height $B$; the same
binomial-calibrated downstream detector as Table~\ref{tab:front} ($L=400$,
$z=3$); $200$ repetitions per cell. The frontier is vertical outside its
resolution band: speed decides, size does not. In the leftmost column the
ramp is long and power grows weakly with size through the number of windows
it offers, the repeated-attempt mechanism of the text; the three lowest
cells there are indistinguishable at this repetition count.}
\label{fig:map}
\end{figure}

\section{Related work}\label{sec:related}

\textbf{Adaptive conformal inference.} The false-alarm guarantee of the
tracker family is the deterministic telescoping bound of \cite{gibbs2021},
which already gives the $1/T$ scale of the long-run deviation; see
\cite{angelopoulos2023} for background. What we add is not the scale but the
exact stochastic law behind it: the pathwise budget $\delta/\eta$ and its
lattice exactness, the ramp rate $c/\eta$ with its self-limiting saturation,
and the closed-form null fluctuation of the windowed rate with its
miscalibration consequence for every downstream detector. That the alarm
indicators are increments of a tight state, so that their windowed sums have
bounded variance, is a spectral-density-at-zero statement familiar in time
series. \cite{gibbs2021} note (Section~4.2.2, Appendix~A.7) that the negative
dependence induced by the update makes the errors concentrate no slower than
an i.i.d.\ Bernoulli sequence, and their Theorem~4.1 accordingly carries a
Hoeffding term on the binomial scale. Theorem~\ref{thm:fluct} shows the null
scale is in fact $1/L$, a factor $\sqrt L$ smaller: the binomial reading is a
valid upper bound and a loose one, by exactly $\sqrt{\eta\varphi(q_0)L}$, and
we have not found that rate drawn out for self-calibrating monitors, nor its
calibration cost quantified. The control-theoretic reading of the
update is made explicit by conformal PID control \cite{angelopoulos2023pid},
which augments the integral recursion with proportional and derivative
terms to improve tracking; the budget and fluctuation identities here
concern the integral core, and the blind-set results bind such variants
unchanged, since Theorem~\ref{thm:confusion} constrains every rule.
Recent work documents power failure modes of
weighted conformal detection under shift \cite{confpower}; the mechanism
there is internal to the weighting scheme, whereas the losses quantified here
are imposed on any monitor by the tolerance obligation itself.

\textbf{Stochastic approximation.} The tracker is Robbins--Monro
\cite{robbinsmonro1951} quantile
estimation at constant step; the stationary variance and AR(1) structure
used in Theorem~\ref{thm:fluct}(ii) are classical for constant-step
stochastic approximation \cite{kushneryin2003}. We have not found in that
literature the pathwise budget $\delta/\eta$, its lattice exactness, the
self-limiting saturation, or the $1/L$ certificate fluctuation law with its
miscalibration consequence; the identity perspective, which derives all of
them from one telescoping lemma, appears to be the new ingredient.

\textbf{Anytime-valid inference.} Game-theoretic and e-value based methods
construct tests whose validity holds uniformly over composite nulls and at
arbitrary stopping times \cite{vovk2005,shafervovk2019,ramdas2023}. These
guarantees concern size; Theorem~\ref{thm:confusion} shows that the same
uniformity, imposed over a drift class, caps power on the difference set for
this family as for any other. The two viewpoints are complementary: anytime
validity disciplines the certificate, the present results price it. Within
this family, contamination of the calibration set by post-change data has
been identified as a cause of diluted evidence and delayed detection, and
removed by testing against a fixed null reference \cite{shaer2026}. That
design presupposes that every shift is anomalous; it is unavailable
precisely when some drifts are legitimate and must be tolerated, which is
the regime studied here.

\textbf{Stealthy attacks in secure control.} The security literature has its
own theory of invisible inputs: attacks exciting only the zero dynamics of a
known plant evade any detector fed by the residual
\cite{pasqualetti2013,teixeira2015}. The mechanism is again an
indistinguishability set, but its origin differs in every structural
respect: there the set is fixed by the plant model, computed offline, and
exploited by an adversary who knows the dynamics; here the set is
indexed by the monitor's own adaptation scale, exists for any score source
without a model, and caps every decision rule, adversarial exploitation
being possible but not required. The two mechanisms compose rather than
compete: an adaptive monitor deployed on a plant with nontrivial zero
dynamics inherits both blind sets.

\textbf{Sequential change detection.} Slowly drifting pre-change regimes are
a known difficulty for classical procedures
\cite{page1954,lorden1971,basseville1993}. That literature typically fixes a
known pre-change law; the composite-null formulation with a drift class, and
the exact difference-set characterization of invisibility, address the
setting where the null itself must be tolerated as a set.

\textbf{Quantitative diagnosability.} The distinguishability framework of
\cite{eriksson2013} measures, in linear Gaussian descriptor models, a
Kullback--Leibler distance between a fault and a fault mode, and proves it
upper-bounds the fault-to-noise ratio of any linear residual generator. Its
fault-free mode is a singleton, so set structure appears only between fault
modes; residual generators there are designed offline and nothing
recalibrates. Our nominal is a set precisely because the monitor chooses to
tolerate, the set is indexed by the monitor's adaptation scale
(Proposition~\ref{prop:manufacture}), and the certificate stays clean during
the collapse. The classical geometric account of undetectable faults
\cite{massoumnia1989,chenpatton1999} is the linear, offline instance of the
difference-set mechanism, fixed by the plant model rather than manufactured
by an adaptive rule.

\textbf{Contraction-based detectability.} Recent work bounds the
distributional separation between nominal and faulty trajectory densities
via stochastic contraction \cite{ibrahim2026}. Those bounds depend on
fault magnitude, assume a known nominal, and analyze batch likelihood
detectors; nothing adapts, so nothing absorbs. The two lines are
complementary: separation bounds say when behaviours differ; the present
results say when an adaptive monitor is forbidden from using the difference.

\textbf{Empirical gradual-drift failures.} Hong \cite{gradual} studies
world-model-based self-monitoring of reinforcement learning agents under
continuous observation drift and reports a sharp absorption threshold whose
existence and sigmoid shape are invariant across three detector families and
model capacities, together with a collapse-before-awareness mode in fragile
environments, described there as fundamentally unmonitorable. These
observations are what identity-level mechanics predict: the existence of the
threshold and its invariance across detector families follow from the budget
and rate identities rather than from any particular detector, and the
unmonitorable mode instantiates the manufactured blind set of
Theorem~\ref{thm:confusion}; its absorption boundary is located exactly by
Theorem~\ref{thm:ramp} at $c=\eta(1-\alpha)$ and displayed in
Figure~\ref{fig:trace}, where a clean certificate accompanies a fault of
arbitrary accumulated size. The reported dependence of the threshold's
position on environment dynamics lives in how the system shapes the score
stream, outside the scope of the present monitor-level analysis.

\section{Discussion}\label{sec:discussion}

\textbf{What a clean certificate certifies.} The empirical alarm rate of a
self-calibrating monitor certifies that its bookkeeping closes. It does not
certify vigilance, and Proposition~\ref{prop:validity} shows it cannot: the
identity holds whatever passes through. Reading such certificates as
evidence of safety amounts to reading an accountant's balanced ledger as
proof that nothing was stolen. The exact identities make the gap
quantitative: a bounded, size-independent budget for steps, exact to within one
alarm in general and exact on every trajectory on the lattice; a rate, not
an amount, for drifts,
valid precisely up to the alarm cap; a certificate whose own null
fluctuation is $1/L$, so that the usual binomial reading of it is
miscalibrated by $\sqrt{\eta\varphi(q_0)L}$; and a blind set, manufactured
by the monitor's tolerance, where the evidence is exactly zero for every
detector.

\textbf{Design consequences.} Three are unconditional. First, the gain
$\eta$ is not a free dial: it indexes the tolerated class and hence the
blind set, in the exact proportion of
Proposition~\ref{prop:manufacture}, and no setting removes both costs.
Second, downstream detectors should be calibrated on the exact fluctuation
law, not the binomial scale: at identical measured false-alarm behaviour
this resolves drifts two to three times slower in our regime, and the gap
grows as $\sqrt L$. Third, certification of adaptive monitors should report,
next to the false-alarm level, the effective tolerated class, the implied
blind set, and the calibration convention of every detector reading the
alarm stream, all computable from design parameters.

\textbf{Limitations.} The identities are exact for the linear-gain tracker
on scalar scores; windowed-quantile and level-update variants obey them to
first order, and multivariate scores are not treated. The fluctuation
constants are linearized in $\eta$ and verified at the percent level; the
pathwise $1/L$ structure itself is exact. The confusion theorem is
worst-case over the drift class and requires the tolerance obligation to be
uniform; average-case detectability under concentrated priors is not
excluded and is quantified by the projection bound. The operational frontier
of any fixed detector has a finite transition band, within which fault
duration matters through repeated windows; the information-theoretic
frontier at twice the tolerated speed is exact.

\textbf{Outlook.} The layer studied here is deliberately abstract: scores
arrive from an unspecified source. In estimation pipelines the score is the
residual of a state observer, and the residual floor is then governed by
system-level quantities such as the observer's convergence margin, which
shape both the effective drift class and the geometry of the blind set. The
composition of the present results with certified learned observers, through
to embedded deployment, is the subject of ongoing work and will be reported
separately, as is the pricing of intermittent ground truth within the same
economy. On the fundamental side,
the conjecture we leave open is a full law of adaptive surveillance: a
characterization, for general monitor families and tolerance classes, of the
exact trade-off surface linking the false-alarm guarantee, the evidence
budget, the certificate's fluctuation scale, and the measure of the
manufactured blind set, together with monitors achieving it.

\appendix
\section{Deferred arguments}\label{app:ramp}

\textbf{Boundedness of the recentred ramp chain.} Let
$p_{t+1}=p_t+\eta(\mathbf 1\{\xi_t>p_t\}-\alpha)-c$ with
$0\le c<\eta(1-\alpha)$ and $\xi_t$ i.i.d.\ with continuous CDF $F$ of full
support. Increments are bounded by $\eta+c$. The conditional mean drift is
$d(p)=\eta(\bar F(p)-\alpha)-c$, continuous and strictly decreasing, with
$d(p)\to\eta(1-\alpha)-c>0$ as $p\to-\infty$ and $d(p)\to-\eta\alpha-c<0$ as
$p\to+\infty$; let $p^\ast$ solve $\bar F(p^\ast)=\alpha+c/\eta$. For any
$\varepsilon>0$ there is $\rho>0$ with $d(p)\le-\rho$ for
$p\ge p^\ast+\varepsilon$ and $d(p)\ge\rho$ for $p\le p^\ast-\varepsilon$.
The function $V(p)=|p-p^\ast|$ has bounded increments and uniformly
negative drift outside a compact interval, and two remarks keep what
follows elementary and free of any irreducibility assumption; the latter
would in fact be \emph{false} in the usual Lebesgue sense, since the chain
has two-valued increments and evolves on a countable orbit (already for
$c=0$ and rational $\alpha=p/q$ in lowest terms it is a countable-state
lattice chain of period $q$, so no aperiodicity is available either; the
pathwise route below is not a stylistic choice).

\emph{(i) Geometric excursions and $p_T/T\to0$.} Bounded increments
together with drift $\le-\rho$ (resp.\ $\ge\rho$) outside $K$ place $V$
under Hajek's drift lemma \cite{hajek1982}: return times to $K$ have
uniformly geometric tails from bounded starting sets, and
$\mathbb E_x[\tau_K]\le V(x)/\rho$ by the basic Foster bound. Writing $G_j$
for the successive excursion lengths away from $K$, geometric tails give
$\sum_j\Pr(G_j>\varepsilon j)<\infty$ for every $\varepsilon>0$, so
$G_j/j\to0$ a.s.\ by Borel--Cantelli; since $|p_t|\le\max_K|\cdot|+
(\eta+c)\,G_j$ throughout the $j$-th excursion, $p_T/T\to0$ almost surely,
and the Ces\`aro limit $\alpha+c/\eta$ follows from the telescoped identity
alone. (See \cite{meyntweedie2009} for the general drift theory.)

\emph{(ii) Every invariant law closes the books.} The occupation measures
$\tfrac1T\sum_{t\le T}\delta_{p_t}$ are tight by (i), so an invariant
probability $\pi$ exists by Krylov--Bogolyubov. No uniqueness is needed for
the exactness claim: integrating the recursion under \emph{any} invariant
$\pi$ gives $0=\mathbb E_\pi[p_{t+1}-p_t]=\eta(\mathbb E_\pi[e_t]-\alpha)-c$,
hence $\mathbb E_\pi[e_t]=\alpha+c/\eta$ for every invariant law.

\textbf{Time-varying leak.} The argument above uses only that the drift of
$V$ is uniformly negative outside a compact set, so it applies verbatim to
the time-inhomogeneous chain $p_t=q_t-\nu_t$ with
$p_{t+1}=p_t+\eta(\mathbf 1\{\xi_t>p_t\}-\alpha)-(\nu_{t+1}-\nu_t)$ whenever
$\sup_t|\nu_{t+1}-\nu_t|\le v<\min(\eta\alpha,\eta(1-\alpha))$: the
conditional drift satisfies
$d_t(p)\le\eta(\bar F(p)-\alpha)+v\le-(\eta\alpha-v)<0$ for $p$ large and
$d_t(p)\ge\eta(1-\alpha)-v>0$ for $p$ small, uniformly in $t$, and
increments stay bounded by $\eta+v$. Hajek's lemma is insensitive to time
inhomogeneity under uniform drift bounds, so $p_T/T\to0$ a.s.\ as before,
and telescoping gives
$\tfrac1T\sum_t e_t=\alpha+\nu_T/(\eta T)+o(1)$ a.s.; since
$|\nu_T|\le vT$, every accumulation point of the Ces\`aro average lies in
$[\alpha-v/\eta,\,\alpha+v/\eta]$. With $v=\varepsilon\eta$ and
$\varepsilon<\min(\alpha,1-\alpha)$ this is the absorption claim of
Proposition~\ref{prop:manufacture}; a constant slope makes the limit exact,
recovering Theorem~\ref{thm:ramp}.

\textbf{Saturation at and above the boundary.} At $c=\eta(1-\alpha)$ the
increments of $p$ are $0$ on alarm steps and $-\eta$ on quiet steps, so $p$
is nonincreasing and moves only at quiet steps. Consequently the value of
$p$ after its $k$-th quiet step is the deterministic quantity $p_0-k\eta$,
independent of when that step occurred; write $q_k=F(p_0-k\eta)$ for the
quiet probability once $k$ quiet steps have occurred. First, $p_t\to-\infty$
a.s.: if instead $p_t\downarrow p_\infty>-\infty$, then eventually
$p_t\le p_\infty+1$, where the quiet probability is bounded below by
$F(p_\infty)>0$ by full support, so infinitely many quiet steps occur a.s.\
by the second Borel--Cantelli lemma, contradicting convergence to a finite
limit; hence $p_t\to-\infty$ and $q_k\to0$.

Second, and this is the step the informal wait-time argument skips: since
$q_k\to0$, the number of alarm steps between the $k$-th and $(k{+}1)$-th
quiet step, $N_k$, is Geometric$(q_k)$, and the $q_k$ are a
\emph{deterministic} sequence (not random, since $p$ after $k$ quiet steps
is exactly $p_0-k\eta$), so the $N_k$ are \emph{independent} across $k$, not
merely conditionally so. Fix $\varepsilon>0$ and $K_\varepsilon$ with
$q_k\le\varepsilon$ for $k\ge K_\varepsilon$; for such $k$, $N_k$
stochastically dominates an independent $\mathrm{Geo}(\varepsilon)$ variable
$\widehat N_k$, and by coupling,
$T_K=\sum_{k<K}(N_k+1)\ge\sum_{K_\varepsilon\le k<K}\widehat N_k$. The
ordinary strong law applied to the i.i.d.\ sequence $(\widehat N_k)$ gives
$\sum_{K_\varepsilon\le k<K}\widehat N_k/K\to(1-\varepsilon)/\varepsilon$
a.s., so $\liminf_K T_K/K\ge(1-\varepsilon)/\varepsilon$ a.s.; letting
$\varepsilon\downarrow0$ gives $T_K/K\to\infty$ a.s., i.e.\ $Q_T/T\to0$ and
the alarm frequency converges to $1$. For $c>\eta(1-\alpha)$ the drift
satisfies $d(p)\le\eta(1-\alpha)-c<0$ everywhere, so $p_t\to-\infty$
linearly and $\bar F(p_t)\to1$; the Ces\`aro alarm rate converges to $1$ in
both cases.

\textbf{Joint recurrence and the coupling time.} We bound $\mathbb E[T_1]$
by a structural reduction followed by renewal--reward. By
Lemma~\ref{lem:coupling}(ii) the gap never exceeds
$G:=\max(g_{t_0},\eta)$, so $v_t\in[u_t,u_t+G]$ for all $t$: the pair lives
in a strip of width $G$ around the diagonal, and it belongs to
$K_G\times K_G$, $K_G:=\{x:\operatorname{dist}(x,K)\le G\}$, whenever the
single coordinate $u_t$ belongs to $K$. Joint recurrence therefore reduces
to marginal recurrence, and no drift condition on the pair is needed. None
is available in the naive form: the sum
$V(u,v)=|u-p^\ast|+|v-p^\ast|$ can have \emph{positive} one-step drift on
the unbounded region where one coordinate idles near $p^\ast$, whose local
contribution is $+2\eta\alpha(1-\alpha)$ at the centre, while the distant
coordinate contributes only $-\eta\min(\alpha,1-\alpha)$; it is the strip
structure, not a joint Lyapunov function, that makes the pair recurrent.
The marginal chain $u$ is the $c=0$ chain above, with bounded increments
and uniformly inward drift outside $K$, so by Hajek's lemma
\cite{hajek1982} its return times to $K$ have uniformly geometric tails,
and $\kappa_1$, the supremum over states reachable in one step from $K$ of
the expected return time to $K$, is finite.

Once $u_t\in K$ and $g_t\ge\eta$, the shrink probability at that step is
$\Pr(\xi_t\in(u_t,v_t])\ge F(u_t+\eta)-F(u_t)\ge
\beta_K:=\inf_{u\in K}\{F(u+\eta)-F(u)\}>0$, a property of the noise law
alone (continuity and full support). A failed attempt costs at most one
step plus the wait until $u$ next visits $K$, so renewal--reward gives an
expected wait per shrink of at most $\kappa_2:=(1+\kappa_1)/\beta_K$. At
most $\lceil g_{t_0}/\eta\rceil$ shrinks are needed, and in the coupling of
Theorem~\ref{thm:budget} the trackers start a distance $\delta$ apart in the
compact set of the statement, so the initial approach to $K$ costs
$\mathbb E\le V(u_{t_0})/\rho=O(g_{t_0})$ by the Foster bound; hence
$\mathbb E[T_1]\le C\lceil g_{t_0}/\eta\rceil$ with $C$ depending only on
the noise law, $\eta$ and $\alpha$.

\textbf{Sharpness of the unit slack.} The gap $g_t=q^f_t-q_t-\delta$
evolves on the lattice $-\delta+\eta\mathbb Z$
(Lemma~\ref{lem:coupling}(i), applied to $u=q^f-\delta$ and $q$, which start
equal). If $\delta\in\eta\mathbb Z$ the lattice contains $0$ and the
coupling is exact from $T_1$ on: the slack is $0$, on every trajectory, as
measured in Table~\ref{tab:ident}. If $\delta\notin\eta\mathbb Z$ the
lattice excludes $0$ and $g_t$ oscillates forever between the two lattice
points bracketing $0$, so the excess mass oscillates within a unit window
around $\delta/\eta$ and the slack of one alarm is attained: it cannot be
improved to $o(1)$ in general, and the measured range $[10.00,11.00]$ for
$\delta=0.53$ exhibits it.

\textbf{Linearized stationary variance.} Writing
$q_{t+1}-q_0=(q_t-q_0)-\eta\varphi(q_0)(q_t-q_0)+\eta m_t+O((q_t-q_0)^2)$
with $m_t=e_t-\mathbb E[e_t\mid q_t]$ a bounded martingale difference of
conditional variance $\alpha(1-\alpha)+O(\eta)$, the linear part is AR(1)
with parameter $\rho=1-\eta\varphi(q_0)$, stationary variance
$\eta^2\alpha(1-\alpha)/(1-\rho^2)=\eta\alpha(1-\alpha)/(2\varphi(q_0))
\cdot(1+O(\eta))$, and lag-$L$ correlation $\rho^L$
\cite{kushneryin2003}.

\end{document}